\documentclass{article}

    \usepackage[preprint,nonatbib]{neurips_2023_arXiv}

\usepackage[utf8]{inputenc} 
\usepackage[T1]{fontenc}    
\usepackage{hyperref}       
\usepackage{url}            
\usepackage{booktabs}       
\usepackage{amsfonts}       
\usepackage{nicefrac}       
\usepackage{microtype}      
\usepackage{xcolor}         

\usepackage{amsmath,amssymb,mathrsfs}
\usepackage{algorithmic}
\usepackage{algorithm}

\usepackage[non-sorted-cites,initials,alphabetic,nobysame]{amsrefs}

\usepackage{amsthm}

\usepackage{cleveref}

\crefname{equation}{equation}{equations}
\crefname{figure}{figure}{figures}
\crefname{table}{table}{tables}
\crefname{algorithm}{Algorithm}{Algorithms}

\crefname{section}{}{}
\creflabelformat{section}{\S #2#1#3}
\crefname{subsection}{}{}
\creflabelformat{subsection}{\S #2#1#3}
\crefname{appendix}{Appendix}{Appendixes}

\crefname{dammy}{}{}
\crefname{definition}{Definition}{Definitions}
\crefname{proposition}{Proposition}{Propositions}
\crefname{lemma}{Lemma}{Lemmas}
\crefname{theorem}{Theorem}{Theorems}
\crefname{corollary}{Corollary}{Corollaries}
\crefname{remark}{Remark}{Remarks}
\crefname{example}{Example}{Examples}

\theoremstyle{definition}
\newtheorem{dammy}{Dammy}[section]
\newtheorem{definition}[dammy]{Definition}
\newtheorem{proposition}[dammy]{Proposition}
\newtheorem{lemma}[dammy]{Lemma}
\newtheorem{theorem}[dammy]{Theorem}
\newtheorem{corollary}[dammy]{Corollary}
\newtheorem{remark}[dammy]{Remark}

\numberwithin{equation}{section}

\DeclareMathOperator*{\argmin}{arg\,min} 
\DeclareMathOperator*{\argmax}{arg\,max} 

\DeclareMathOperator{\Proj}{Proj}

\newcommand{\dic}{\mathrm{dic}}

\title{A proof of imitation \\ of Wasserstein inverse reinforcement learning \\ for multi-objective optimization}

\author{%
    Akira Kitaoka \\
    NEC Corporation \\
    \texttt{akira-kitaoka@nec.com}
    \And
    Riki Eto \\
    NEC Corporation \\
    \texttt{riki.eto@nec.com} \\
}

\begin{document}

\maketitle

\begin{abstract}
    We prove Wasserstein inverse reinforcement learning enables the learner's reward values to imitate the expert's reward values in a finite iteration for multi-objective optimizations.

    Moreover, we prove Wasserstein inverse reinforcement learning enables the learner's optimal solutions to imitate the expert's optimal solutions for multi-objective optimizations with lexicographic order.
\end{abstract}


\section{Introduction}

Artificial intelligence (AI) has been used to automate various tasks recently. Generally,  automation by AI is achieved by setting an index of goodness or badness (reward function) of a target task and having AI automatically search for a decision, that is, an optimal solution in mathematical optimization that maximizes or minimizes the index. 
For example, in work shift scheduling (e.g. \cites{Cheang-2003-nurse,Graham-1979-optimization}), which is a type of combinatorial optimization or multi-objective optimization, we can create shifts that reflect our viewpoints by calculating the optimal solution of a reward function that reflects our intentions for several viewpoints, such as ``degree of reflection of vacation requests,'' ``leveling of workload,'' and ``personnel training,'' and so on while preserving the required number of workers, required skills, labor rules.
However, setting the reward function, i.e., "what is optimal?", manually requires a lot of trial-and-error, which is a challenge for the actual application of mathematical optimization.
Creating a system that can solve this problem automatically is essential in freeing the user from manually designing the reward function.

Inverse reinforcement learning (IRL) 
\cites{Russell-1998-learning,Ng-2000-algorithms} is generally known as facilitating the setting of the reward function. In IRL, a reward function that reflects expert's intention is generated by learning expert's trajectories, iterating optimization using the reward function, and updating the parameters of the reward function.
In IRLs which is fomulated by Ng and Russell \cite{Ng-2000-algorithms}, and Abbeel and Ng \cite{Abbeel-2004-apprenticeship}, in multi-objective optimization, the space of actions, i.e., the space of optimization results, is enormous.
In other words, it is necessary to set the reward function for the space of actions and states, which is computationally expensive.

Maximum entropy IRL (MEIRL) \cite{Ziebart-2008-maximum}
and 
guided cost learning (GCL) \cite{Finn-2016-guided}
are 
methods to adapt IRL to multi-objective optimization problems.
However, these methods have their issues.
For example, MEIRL requires the sum of the reward functions for all trajectories to be computed. This makes maximum entropy IRL computationally expensive. On the other hand, GCL approximates the sum of the reward functions for all trajectories by importance sampling. However, since multi-objective optimization problems take discrete values, it is difficult to find the probability distribution corresponding to a given value when a specific value is input. One reason for this difficulty is that in multi-objective optimization problems, even a small change in the value of the reward function may result in a large change in the result.

Eto proposed IRL for multi-objective optimization including combinatorial optimization, Wasserstein inverse reinforcement learning (WIRL) \cite{Eto-2022}, inspired by Wasserstein generative adversarial networks \cite{Arjovsky-Chintala-Bottou-17}. 
In multi-objective optimization problems, WIRL makes it possible to learn a reward function that reflects the expert's decision-making data, i.e., the expert's intentions.

For multi-objective optimization, Kitaoka and Eto showed WIRL is convergent \cite{Kitaoka-2023-convergence-IRL}.
However, when WIRL is convergent, there is no known proof that the learner's reward functions and actions imitate the expert's reward functions and actions.
Eto proposed that we do inverse reinforcement learning for multi-objective optimizations with WIRL \cite{Eto-2022}, although there was no theoretical explanation for this phenomenon.

In this paper, 
we show that if WIRL for multi-objective optimization is convergent, then the learner's reward values converges to the expert's reward values.
Moverover, we prove that when WIRL is convergent for multi-objective optimization, the learner's actions coincide with the expert's actions.
In \cref{sec:WIRL}, we recall the definition of WIRL.
In \cref{sec:intention_learning_over_inner_product_space}, we recall the definition and propositions of WIRL to multi-objective optimizations.
In \cref{sec:imitation_of_intention_wrt_reward},
we show that if WIRL for multi-objective optimization is convergent, then the learner's reward values converge to the expert's reward values.
In \cref{sec:imitation_of_intention_wrt_action},
we show when WIRL is convergent for multi-objective optimization, the learner's actions coincide with the expert's actions.

\section{Wasserstein inverse reinforcement learning}\label{sec:WIRL}

Let $\mathcal{H} , \mathcal{H}_{\mathcal{S}} $ be inner product spaces,
$\mathcal{S} \subset \mathcal{H}_{\mathcal{S}}$ be a space of states
$\mathcal{A} \subset \mathcal{H}$ be a space of actions,
$\mathcal{T} := \prod_k \left( \mathcal{S} \times \mathcal{A} \right)$ be a space of trajectories.
Let $\Theta \subset \mathcal{H}$,
and we call $\Theta$ a space of feature maps.
Let $\Phi \subset \mathcal{H}$,
and we call $\Phi $ a space of parameters of learner's trajectories.
Let $f_{\bullet } \colon \mathcal{T} \to \Theta$ be 1-Lipschitz,
and we call $f_{\bullet }$ the feature map. 
For any Lipschitz function $r_{\theta} \colon \mathcal{T} \to \mathbb{R}$,
the norm of Lipschitz $\| r_{\theta} \|_L$ is defined by
\[
    \| r_{\theta} \|_L 
    := 
    \sup_{\tau_1 \not = \tau_2 } \frac{ | r_{\theta} (\tau_1) - r_{\theta} (\tau_2) | }{ \| \tau_1 - \tau_2 \| }
    .
\]
Let $\delta_x$ be the Delta function at $x$.
Let ${\{ \tau_{E}^{(n)} \}}_{n=1}^N$ be the data of expert's trajectories,
and we define the distribution of expert's trajectories by
\[
    \mathbb{P}_E := \frac{1}{N} \sum_{n=1}^N \delta_{\tau_{E}^{(n)}}
    .
\]
With the initial state $s_{\mathrm{ini}}^{(n)}$ of expert's trajectory $\tau_{E}^{(n)}$,
and the generator $g_{\bullet} (\bullet ) \colon \Phi \times \mathcal{S} \to \mathcal{T} $ of learner's trajectory, 
we define the distribution of learner's trajectories by
\[
    \mathbb{P}_\phi := \frac{1}{N} \sum_{n=1}^N \delta_{g_{\phi} (s_{\mathrm{ini}}^{(n)})}   
    .
\]
The Wasserstein distance between the distribution $\mathbb{P}_E$ of expert's trajectories and that $\mathbb{P}_\phi$ of learner's trajectories 
is, with the Kantrovich-Rubinstein duality (c.f. \cite{Villani-2009}),
\[
    W ( \mathbb{P}_E , \mathbb{P}_\phi )
    =
    \sup_{\| r_{\theta} \|_L \leq 1 }
    \left\{
        \frac{1}{N} \sum_{n=1}^N r_{\theta} (\tau_{E}^{(n)})
        -
        \frac{1}{N} \sum_{n=1}^N r_{\theta} (g_{\phi} (s_{\mathrm{ini}}^{(n)}))
    \right\}
    ,
\]
where 
$r_{\theta}$ is 1-Lipschitz function.

We are interested in finding $\phi \in \Phi$ satisfying the following problem:
\begin{equation}
    \argmin_{\phi \in \Phi} W ( \mathbb{P}_E , \mathbb{P}_\phi )
    .
    \label{eq:WIRL-origin}
\end{equation}
With 
\begin{equation}\label{eq:sim_1-Lip}
    \left\{ r_{\theta}(\tau) := \theta^{\intercal} f_{\tau} \, \middle| \, \theta \in \Theta \right\}
    \text{ insted of }
    \{ \| r_{\theta} \|_L \leq 1 \}
    ,
\end{equation}
to find $\phi \in \Phi$ satisfying \cref{eq:WIRL-origin}
can be roughly replaced by finding 
\begin{equation}
    \argmin_{\phi \in \Phi} 
    \sup_{ \theta \in \Theta }
    \left\{
        \frac{1}{N} \sum_{n=1}^N \theta^{\intercal} f_{\tau_{E}^{(n)}}
        -
        \frac{1}{N} \sum_{n=1}^N \theta^{\intercal} f_{g_{\phi} (s_{\mathrm{ini}}^{(n)})}
    \right\}
    .
    \label{eq:WIRL-origin-2}
\end{equation}
By changing the sign,
we may consider solving
\begin{equation}
    \argmax_{\phi \in \Phi} 
    \inf_{ \theta \in \Theta }
    \left\{
        \frac{1}{N} \sum_{n=1}^N \theta^{\intercal} f_{g_{\phi} (s_{\mathrm{ini}}^{(n)})}
        -
        \frac{1}{N} \sum_{n=1}^N \theta^{\intercal} f_{\tau_{E}^{(n)}}
    \right\}
    .
    \label{eq:WIRL-origin-3}
\end{equation}
The IRL that solves \cref{eq:WIRL-origin-2} or \cref{eq:WIRL-origin-3},
is called Wasserstein inverse reinforcement learning (WIRL) \cite{Eto-2022}.

\begin{remark}
    In this paper, learning to maximize the reward function of a history-dependent policy is called reinforcement learning.
    Learning that minimizes the score between the reward function calculated from the expert's trajectory and the reward function learned by reinforcement learning is called inverse reinforcement learning.
\end{remark}

\section{WIRL for multi-objective optimization}
\label{sec:intention_learning_over_inner_product_space}

We adapt WIRL to multi-objective optimization.
Let $\mathcal{H}^{\prime}$ be an inner product space,
$\mathcal{A}^{\prime}$ be a set such that $\mathcal{A}^{\prime} \subset \mathcal{H}^{\prime}$,
$h \colon \mathcal{A}^{\prime} \to \mathcal{H} $ be a continuous function.
Let $X(s)$ be a compact set\footnote{
If $\mathcal{A}^{\prime}$ is in the Euclid space,
compact sets are bounded closed sets.   
}
in $\mathcal{A}^{\prime}$ for $s \in \mathcal{S}$. 
We set the space of trajectories $\mathcal{T} = \mathcal{S} \times \mathcal{A}$.
Then, multi-objective optimization (e.g. \cites{Murata-1996-multi,Gunantara-2018-review}) is to solve for the following optimization:
\begin{equation}\label{eq:optimal_solver_classic}
    a (\phi , s )  \in \argmax_{h(x) \in h ( X (s))} \phi^{\intercal} h ( x ) 
    .
\end{equation}
We call the solution or the learner's action $a (\phi , s ) $ the solver.
For $\phi \in \Phi$ and an action $a \in \mathcal{A}$, we call $\phi^{\intercal} a$ the reward value.

We set the feature map $f= \Proj_{\mathcal{A}}$,
where $\Proj_{\mathcal{A}} \colon \mathcal{T} \to \mathcal{A}$ is the projection from $\mathcal{T}$ to $\mathcal{A}$. 
We define the generator $g_{\phi} (s)$ by 
\begin{equation*}
    g_{\phi} (s) := (s , a (\phi , s )).
\end{equation*}
We say that
\textbf{intention learning} with WIRL
is 
the result of applying WIRL to the above setup.

The expert's action $a^{(n)}$ is assumed to follow an optimal solution.
Namely, we often run WIRL intention learning by assuming that there exists some $\phi_0 \in \Phi $ and that we can write $a^{(n)} = a (\phi_0 , s^{(n)} )$.

\begin{remark}
    Examples of adapting intention learning to linear and quadratic programming are described in \cite{Kitaoka-2023-convergence-IRL}*{\S 5}.
\end{remark}

We give the inverse propblem of the multi-objective optimization problem that is equivalent to the problem handled by intention learning with WIRL.
\begin{definition}\label{defi:intention-learning-IOP}
    {
        \rm (\cite{Kitaoka-2023-convergence-IRL}*{Definition 4.4})
    }
    Let $\mathcal{H} , \mathcal{H}_{\mathcal{S}}, \mathcal{H}^{\prime} $ be inner product spaces,
    $\mathcal{S} \subset \mathcal{H}_{\mathcal{S}}$,
    $\mathcal{A}^{\prime} \subset \mathcal{H}^{\prime}$,
    $\Phi \subset \mathcal{H}$ be a closed convex set 
    $h \colon \mathcal{A}^{\prime} \to \mathcal{H} $ be the continuous function,
    $X(s) \subset \mathcal{A}^{\prime}$ be a compact non-empty set for $s \in \mathcal{S}$.

    Then, the inverse problem of multi-objective optimization problem (IMOOP)
    for the solver $a (\phi , s )$ and
    trajectories of an expert $\{ \tau_E^{(n)} = ( s^{(n)}, a^{(n)}) \}_n \subset \mathcal{H}_{\mathcal{S}}\times \mathcal{H}$
    is to find $\phi \in \Phi $ satisfying
    \begin{equation}
        \text{ minimize } 
        F(\phi) :=\frac{1}{N} \sum_{n=1}^N \phi^{\intercal} a (\phi , s^{(n)} )
            - \frac{1}{N} \sum_{n=1}^N \phi^{\intercal} a^{(n)}
        ,
        \quad
        \text{ subject to }
        \phi \in \Phi
        .
        \label{eq:intention-learning-IOP}
    \end{equation}
\end{definition}

\begin{proposition}
    {
        \rm (\cite{Kitaoka-2023-convergence-IRL}*{Lemma 4.6})
    }
    In the setting of 
    $\Theta = \Phi$,
    \cref{eq:intention-learning-IOP} is
    the replacement of $\max_{\phi \in \Phi}$ and $\inf_{\theta \in \Theta}$
    in \cref{eq:WIRL-origin-3}, that is,
    \begin{align}
        & 
        \min_{\phi \in \Phi}
        \left\{
            \frac{1}{N} \sum_{n=1}^N \phi^{\intercal} a (\phi , s^{(n)} )
            - \frac{1}{N} \sum_{n=1}^N \phi^{\intercal} a^{(n)}
        \right\}
        \notag \\
        & =
        \min_{ \theta \in \Phi }
        \max_{\phi \in \Phi} 
        \left\{
            \frac{1}{N} \sum_{n=1}^N \theta^{\intercal} a (\phi ,s^{(n)})
            -
            \frac{1}{N} \sum_{n=1}^N \theta^{\intercal} a^{(n)}
        \right\}
        .
        \notag
    \end{align}
\end{proposition}
The subgradient of $F$ is given by the following proposition:
\begin{proposition}
    \label{prop:WIRL-projected_subgradient_method}
    {
        \rm (\cite{Kitaoka-2023-convergence-IRL}*{Lemma 4.8})
    }
    In the setting of \cref{defi:intention-learning-IOP},
    one of the subgradient of $F$ at $\phi \in \Phi$ is
    \[
        \frac{1}{N} \sum_{n=1}^N a (\phi , s^{(n)} )
        - \frac{1}{N} \sum_{n=1}^N a^{(n)}
        .
    \]

\end{proposition}

The algorithm of WIRL for multi-objective optimization
is given by \cref{alg:intention-WIRL-gradual-decay}.
\begin{figure}[ht]
    \begin{algorithm}[H]
        \caption{Intention learning (with WIRL) \cite{Kitaoka-2023-convergence-IRL}*{Algorithm 1}}\label{alg:intention-WIRL-gradual-decay}
        \begin{algorithmic}[1] 
            \STATE initialize $ \phi_1 \in \Phi$
            \FOR{$k =1 , \ldots , K-1$}
                \STATE $\phi_{k+1} \leftarrow \phi_k 
                - \frac{\alpha_k}{N} \sum_{n=1}^N 
                \left(
                    a ( \phi_k , s^{(n)} )
                    -
                    a^{(n)}
                \right)
                $
                \STATE projection onto $\Phi$ for $\phi_{k+1}$
            \ENDFOR
            \RETURN $ \phi_K^{\mathrm{best}} \in \argmin_{\phi_k \in \{ \phi_k \}_{k=1}^K} F(\phi_k)$
        \end{algorithmic}
    \end{algorithm}
\end{figure}

\begin{proposition}\label{prop:IOP_coninsides_with_Intention_L}
    {
        \rm (\cite{Kitaoka-2023-convergence-IRL}*{Lemma 4.11})
    }
    In the setting of \cref{defi:intention-learning-IOP},
    the algorithm which solves IMOOP for the solver $a (\phi , s )$
    coninsides with \cref{alg:intention-WIRL-gradual-decay}.
    Here, $\{ \alpha_k \}_k$ is a nonsummable
    diminishing learning rate, that is,
    \[
        \lim_{k \to \infty } \alpha_k = 0 , 
        \quad
        \sum_{k=1}^{\infty} \alpha_k = \infty
        .    
    \]
\end{proposition}
As a natural question from 
\cref{prop:WIRL-projected_subgradient_method,prop:IOP_coninsides_with_Intention_L},
when the WIRL is close to completion or a subgradient $\frac{1}{N} \sum_{n=1}^N a (\phi , s^{(n)} )
- \frac{1}{N} \sum_{n=1}^N a^{(n)}$ is $0$,
whether the learner's reward values and actions imitate the expert's.

\section{Imitation of intention learning concerning reward value}\label{sec:imitation_of_intention_wrt_reward}

In this section, we show that intention learning enables the learner to imitate reward values that reflects the expert's intentions.

\begin{theorem}\label{theo:intention_learning_complete_reward}
    Let $ \mathcal{H}_{\mathcal{S}}, \mathcal{H}^{\prime} $ be inner product spaces,
    $\mathcal{S} \subset \mathcal{H}_{\mathcal{S}}$,
    $\mathcal{A}^{\prime} \subset \mathcal{H}^{\prime}$,
    $\Phi \subset \mathbb{R}^d$ be a closed convex set, 
    $h \colon \mathcal{A}^{\prime} \to \mathbb{R}^d $ be the continuous function,
    $X(s) \subset \mathcal{A}^{\prime}$ be a compact non-empty set for $s \in \mathcal{S}$.
    We assume that there exists $\phi_0 \in \Phi$ such that 
    $a^{(n)} = a(\phi_0 , s^{(n)})$ for any $n$. 
    Let $\varepsilon > 0$.

    Then, if
    \[
        F (\phi) < \varepsilon   
        , 
    \]
    then for any $n$, we have
    \[
        0
        \leq 
        \phi^{\intercal} a (\phi , s^{(n)})
        -
        \phi^{\intercal} a (\phi_0 , s^{(n)} ) 
        <
        \varepsilon N 
        .
    \]
\end{theorem}

\begin{proof}
    By the definition of the solver \cref{eq:optimal_solver_classic},
    we note that 
    $a (\phi , s^{(n)} ) \in h ( X (s^{(n)}))$.
    By the definition of the solver \cref{eq:optimal_solver_classic},
    we obtain
    \[
        \phi^{\intercal} a (\phi , s^{(n)})
        \geq
        \phi^{\intercal} a (\phi_0 , s^{(n)} )
        . 
    \]
    With the above inequality, we see
    \begin{align*}
        F(\phi ) 
        & = \frac{1}{N} \sum_{n=1}^N \phi^{\intercal} a (\phi , s^{(n)} )
            - \frac{1}{N} \sum_{n=1}^N \phi^{\intercal} a^{(n)}
        \\
        & = \frac{1}{N} \left(
            \phi^{\intercal} a (\phi , s^{(n)})
            - \phi^{\intercal} a (\phi_0 , s^{(n)} ) 
        \right)
        .
    \end{align*}
    Therefore 
    if $F(\phi ) < \varepsilon $, then
    \[
        \phi^{\intercal} a (\phi , s^{(n)})
        - \phi^{\intercal} a (\phi_0 , s^{(n)} ) 
        < \varepsilon N 
        .   
    \]
    
\end{proof}

If there exists $\phi_0 \in \Phi$ so that $a^{(n)} = a(\phi_0 , s^{(n)})$ for any $n$,
then 
\[
    \min_{\phi \in \Delta} F(\phi) = 0 
\]
Kitaoka and Eto showed that the intention learning with WIRL is covnergence \cite{Kitaoka-2023-convergence-IRL}.
\begin{proposition}
    \label{prop:WIRL-convergence-theorem}
    {
        \rm (\cite{Kitaoka-2023-convergence-IRL}*{Theorem 4.12})
    }
    In the setting of \cref{theo:intention_learning_complete_reward},
    we assume that $F$ has the minimum on $\Phi$.

    Then, a sequence $\{ \phi_k^{\mathrm{best}} \}_k$ calculated by the intention learning with WIRL has the following property: 
    for any $\varepsilon >0$
    there exists an natural number $K$ so that for any integer $k>K$,
    \[
        F(\phi_k^{\mathrm{best}})  < \varepsilon
        .    
    \]
    
\end{proposition}

To combine \cref{theo:intention_learning_complete_reward} and \cref{prop:WIRL-convergence-theorem},
we obtain the following corollary:
\begin{corollary}\label{cor:intention_learning_complete_reward_2}
    In the setting of \cref{prop:WIRL-convergence-theorem},
    a sequence $\{ \phi_k^{\mathrm{best}} \}_k$ calculated by the intention learning with WIRL has the following property: 
    for any $\varepsilon >0$
    there exists an natural number $K$ so that for any integer $k>K$,
    \[
        0
        \leq 
        {\phi_k^{\mathrm{best}}}^{\intercal} a (\phi_k^{\mathrm{best}} , s^{(n)})
        -
        {\phi_k^{\mathrm{best}}}^{\intercal} a (\phi_0 , s^{(n)} ) 
        <
        \varepsilon N
        .
    \]
\end{corollary}

\cref{cor:intention_learning_complete_reward_2} means that intention learning enables the learner's reward values to imitate the expert's reward values in linear and quadratic programming problems, integer programming problems, mixed integer programming problems, and so on.

\section{Imitation of intention learning concerning action}\label{sec:imitation_of_intention_wrt_action}

In this section, we set $\mathcal{H} = \mathbb{R}^d$, the $d$-dimensional Euclid space.
Before showing the imitation of intention learning concerning the action, we change the definition of the solver $a (\phi , s )$: 
\begin{equation}\label{eq:optimal_solver}
    a (\phi , s )  := \min_{\dic} \argmax_{h(x) \in h ( X (s))} \phi^{\intercal} h ( x ) 
    ,
\end{equation}
where $\min_{\dic}$ returns to the minimal of the lexicographical order $\leq_{\dic}$.\footnote{
    For $x , y \in \mathbb{R}^d$,
    we define $x \leq_{\dic} y$ if and only if 
    there exists $1 \leq k\leq d$ such that for any $1 \leq i \leq k-1$, $x_i = y_i $ and
    $x_k \leq y_k $.
    We call the order $\leq_{\dic}$ the lexicographical order.
}\footnote{
    Let $B \subset \mathbb{R}^d$.
    The element $b \in \mathbb{R}^d$ is the minimum of $B$ of the lexicographical order $( \mathbb{R}^d , \leq_{\dic} )$,
    if and only if 
    for any $x \in B$, $b \leq_{\dic} x$.
    We set $\min_{\dic}B := b$.
    
    For example, we set $B= \{ (0,0) , (1,-1 ) , (-1,1) \}$.
    To compare the first component, we obtain $\min_{\dic} B = (-1,1)$.
}

\begin{remark}
    In \cites{Eto-2022,Kitaoka-2023-convergence-IRL} and \cref{sec:imitation_of_intention_wrt_reward}
    we define the learner's action, the solver by
    \[ 
        a (\phi , s ) \in \argmax_{h(x) \in h ( X (s))} \phi^{\intercal} h ( x )   
        . 
    \]
    To prove \cref{theo:intention_learning_complete}, which we discuss later, we use lexicographic order in \cref{eq:optimal_solver} to define the learner's actions.

    For practical purposes, it is also conceivable to output only one solution when running multi-objective optimization. As one of the solutions, it is natural to choose the smallest one in the sense of lexicographic order.

\end{remark}

We show that intention learning enables the learner to imitate an action that reflects the expert's intentions:
\begin{theorem}\label{theo:intention_learning_complete}
    Let $ \mathcal{H}_{\mathcal{S}}, \mathcal{H}^{\prime} $ be inner product spaces,
    $\mathcal{S} \subset \mathcal{H}_{\mathcal{S}}$,
    $\mathcal{A}^{\prime} \subset \mathcal{H}^{\prime}$,
    $\Phi \subset \mathbb{R}^d$ be a closed convex set, 
    $h \colon \mathcal{A}^{\prime} \to \mathbb{R}^d $ be the continuous function,
    $X(s) \subset \mathcal{A}^{\prime}$ be a compact non-empty set for $s \in \mathcal{S}$.
    We assume that there exists $\phi_0 \in \Phi$ such that 
    $a^{(n)} = a(\phi_0 , s^{(n)})$ for any $n$. 

    Then, for $\phi \in \Phi $, the following are equivalent:
    \begin{enumerate}
        \item[(1)]
        The subgradient of $F(\phi)$ at $\phi \in \Phi$
        is 
        \[
        \sum_{n=1}^N \left( 
            a (\phi , s^{(n)})
            -
            a (\phi_0 , s^{(n)} ) 
        \right)
        = 0.
        \]
        \item[(2)]
        For any $n$, $g_{\phi} (s^{(n)}) = g_{\phi_0} (s^{(n)})$,
        that is,
        $a ( \phi , s^{(n)}) = a (\phi_0 , s^{(n)})$.
        \item[(3)]
        $W ( \mathbb{P}_{\phi} , \mathbb{P}_{\phi_0} ) = 0$.
    \end{enumerate}
\end{theorem}

From the equivalence of (1) and (2) in \cref{theo:intention_learning_complete},
the completion of intention learning implies that the learner's actions perfectly imitate the expert's in linear programming, quadratic programming, etc.

\begin{lemma}\label{lem:enough_condition_Theta}
    A sufficient condition for a function $r_{\theta}(\tau) := \theta^{\intercal} f_{\tau}$ to be 1-Lipschitz for $\tau$ is 
    \[
        \| \theta \| \leq 1 / \| f \|_{L}
        .
    \] 
\end{lemma}

\begin{proof}
    A sufficient condition for the function $r_{\theta} (\tau)$ to be 1-Lipschitz for $\tau$ is 
    \[
        \frac{
            | \theta^{\intercal} f_{\tau_1} - \theta^{\intercal} f_{\tau_2} |
        }{
            \| \tau_1 - \tau_2 \|
        }
        \leq
        1
        .
    \]
    By Cauchy-Schwarz's inequality, we have
    \[
        | \theta^{\intercal} f_{\tau_1} - \theta^{\intercal} f_{\tau_2} |
        \leq 
        \| \theta \| 
        \| f_{\tau_1} - f_{\tau_2} \|
        .
    \]
    If
    \[
        \| \theta \| 
        \frac{
            \| f_{\tau_1} - f_{\tau_2} \|
        }{
            \| \tau_1 - \tau_2 \|
        }
        \leq 
        1
        ,
    \]
    then $r_{\theta} (\tau)$ is 1-Lipschitz for $\tau$.
    Therefore, to apply $\sup_{\tau_1 \not = \tau_2}$ to both side,
    we obtain the sufficient condition for the function $r_{\theta} (\tau)$ to be 1-Lipschitz for $\tau$,
    \[
        \| \theta \| 
        \| f \|_L
        \leq 1
        .
    \]
\end{proof}

\begin{proof}[\textbf{proof of \cref{theo:intention_learning_complete}}]
    
    (2) $\Rightarrow $ (3)
    We assume that $g_{\phi} (s^{(n)}) = g_{\phi_0} (s^{(n)}) $ for $n$.
    Then, 
    \[
        W ( \mathbb{P}_{\phi} , \mathbb{P}_{\phi_0} ) 
        =
        \sup_{\| r_{\theta} \|_L \leq 1 }
        \left\{
            \frac{1}{N} \sum_{n=1}^N r_{\theta} (g_{\phi} (s_{\mathrm{ini}}^{(n)}))
            -
            \frac{1}{N} \sum_{n=1}^N r_{\theta} (g_{\phi_0} (s_{\mathrm{ini}}^{(n)}))
        \right\}
        =
        \sup_{\| r_{\theta} \|_L \leq 1 }
        \left\{
            0
        \right\}
        = 0
        .
    \]

    (3) $\Rightarrow$ (1)
    For the feature map $f$, we assume thet
    \begin{equation}
        \sum_{n=1}^N \left( 
                f_{g_{\phi} (s_{\mathrm{ini}}^{(n)})}
                -
                f_{\tau_E^{(n)}} 
            \right)
        \not = 0
        \label{eq:3to1}
    \end{equation}
    Since there exists $n$ such that
    \[
        f_{g_{\phi} (s_{\mathrm{ini}}^{(n)})}
                -
                f_{\tau_E^{(n)}}
        \not = 0
        ,
    \]
    we see
    \[
        0 <  
        \frac{\left\|
            f_{g_{\phi} (s_{\mathrm{ini}}^{(n)})}
            -
            f_{\tau_E^{(n)}} 
        \right\|
        }{
            \left\|
                g_{\phi} (s_{\mathrm{ini}}^{(n)})
                -
                \tau_E^{(n)}
            \right\|
        } 
        \leq 
        \| f \|_L ,    
    \]
    i.e., $\| f \|_L \not = 0$.
    We take 
    \begin{align}
        \theta^*
        & :=
        \argmax_{ \| \theta \| \leq 1 / \| f \|_L }
        \theta^{\intercal} 
        \frac{1}{N}
            \sum_{n=1}^N \left( 
                f_{g_{\phi} (s_{\mathrm{ini}}^{(n)})}
                -
                f_{\tau_E^{(n)}} 
            \right)
        =
        \frac{1}{ \| f \|_L }
        \frac{
            \frac{1}{N}
            \sum_{n=1}^N \left( 
                f_{g_{\phi} (s_{\mathrm{ini}}^{(n)})}
                -
                f_{\tau_E^{(n)}} 
            \right)
        }{
            \left\|         
            \frac{1}{N}
            \sum_{n=1}^N \left( 
                f_{g_{\phi} (s_{\mathrm{ini}}^{(n)})}
                -
                f_{\tau_E^{(n)}} 
            \right) 
            \right\|
        } \notag
        .
    \end{align}
    From \cref{eq:sim_1-Lip} and \cref{lem:enough_condition_Theta},
    \begin{align}
        W ( \mathbb{P}_{\phi_0} , \mathbb{P}_\phi ) 
        & \geq
        \sup_{ \| \theta \| \leq 1 / \| f \|_L }
        \left\{
            \theta^{\intercal} \left(
                \frac{1}{N}
                \sum_{n=1}^N \left( 
                    f_{g_{\phi} (s^{(n)})}
                    -
                    f_{g_{\phi_0} (s^{(n)})} 
                \right)
            \right)
        \right\}
        \notag
        \\
        & =
        \theta^{* \intercal} 
        \frac{1}{N}
            \sum_{n=1}^N \left( 
                f_{g_{\phi} (s_{\mathrm{ini}}^{(n)})}
                -
                f_{\tau_E^{(n)}} 
            \right)
        \notag
        \\
        & =
        \frac{1}{\| f \|_L }
        \left\| 
            \frac{1}{N}
            \sum_{n=1}^N \left( 
                f_{g_{\phi} (s_{\mathrm{ini}}^{(n)})}
                -
                f_{\tau_E^{(n)}} 
            \right)
        \right\|
        \notag
        .
    \end{align}
    From the assumption (3), we have 
    \[
        0 \leq
        \frac{1}{\| f \|_L }
        \left\| 
            \frac{1}{N}
            \sum_{n=1}^N \left( 
                f_{g_{\phi} (s_{\mathrm{ini}}^{(n)})}
                -
                f_{\tau_E^{(n)}} 
            \right)
        \right\| \leq W ( \mathbb{P}_{\phi_0} , \mathbb{P}_{\phi} ) = 0
        .
    \]
    Therefore,
    \begin{equation}
        \sum_{n=1}^N \left( 
                f_{g_{\phi} (s_{\mathrm{ini}}^{(n)})}
                -
                f_{\tau_E^{(n)}} 
            \right)
        = 0
        .
        \label{eq:3to1-2}
    \end{equation}
    It contradicts \cref{eq:3to1}.

    Substituting $f=\Proj_{\mathcal{A}}$ for \cref{eq:3to1-2},
    we get
    \begin{equation*}
        \sum_{n=1}^N \left( 
            a (\phi , s^{(n)} ) 
            -
            a (\phi_0 , s^{(n)})
        \right)
        = 0
        .
    \end{equation*}

    (2) $\Rightarrow$ (3)
    We assume that the subgradient of $F$ is given by
    \begin{equation*}
        \sum_{n=1}^N \left( 
            a (\phi , s^{(n)} ) 
            -
            a (\phi_0 , s^{(n)})
        \right)
        = 0
    \end{equation*}
    To act $\phi_0^{\intercal}$ on the both side,
    we have
    \begin{equation*}
        \sum_{n=1}^N \left( 
            \phi_0^{\intercal} a (\phi , s^{(n)} ) 
            -
            \phi_0^{\intercal} a (\phi_0 , s^{(n)})
        \right)
        = 0
        .
    \end{equation*}
    Since 
    by the definition of the solver $a(\phi , s)$,
    \begin{equation*}
        \phi_0^{\intercal} a (\phi_0 , s^{(n)} )
        \geq
        \phi_0^{\intercal} a (\phi , s^{(n)} )
        ,
    \end{equation*}
    for all $n$, we have 
    \begin{equation*}
        \phi_0^{\intercal} a (\phi_0 , s^{(n)} )
        =
        \phi_0^{\intercal} a (\phi , s^{(n)} )
        .
    \end{equation*}
    By the definition of the solver $a(\phi , s)$,
    we see
    \begin{equation*}
        a (\phi , s^{(n)} ) \in \argmax_{h(x) \in h( X (s^{(n)}) ) } \phi_0^{\intercal} h(x)
    \end{equation*}
    Therefore, we obtain
    \begin{equation*}
        a (\phi_0 , s^{(n)} ) \leq_{\dic} a (\phi , s^{(n)} )
        .
    \end{equation*}

    As the same way, to replace to $\phi_0$ and $\phi$,
    we obtain
    \begin{equation*}
        a (\phi , s^{(n)} ) \leq_{\dic} a (\phi_0 , s^{(n)} )
        .
    \end{equation*}
    Summing up, we have 
    \begin{equation*}
        a (\phi , s^{(n)} ) 
        =
        a ( \phi_0 , s^{(n)} ) 
        .
    \end{equation*}

\end{proof}

\section{Related work}

\subsection*{Maximal entropy inverse reinforcement learning}
Ho and Ermon showed that MEIRL is the inverse problem of maximum entropy reinforcement learning \cite{Ho-Ermon-16-GAIL}*{Corollary 3.2.1} .
Significant differences exist between the MEIRL setup used by GAIL and the WIRL setup.
First, they differ in the design of the reward function: MEIRL uses an entropy-regularized value function as the reward function for maximum entropy reinforcement learning, whereas WIRL uses a multi-objective optimization objective function as the reward function.
Second, the settings of state space and action space are different.
\cite{Ho-Ermon-16-GAIL} assumes that the state space and action space are finite sets.
In WIRL, on the other hand, the state space and action space are allowed to be both finite and infinite sets.
Therefore, the argument in \cite{Ho-Ermon-16-GAIL} that measures are replaced by occupancy measures and attributed to Lagrange's undetermined multiplier method for occupancy measures and cost functions cannot be applied to multi-objective optimization.

\section{Conclusion}

\subsection*{Intention learning concerning reward}

If the generator $g_{\phi}$ represents the expert's action, then when WIRL converges for multi-objective optimization,
we show \cref{theo:intention_learning_complete_reward}, which claims the learner's reward values are convergent to the expert's.
On the other hand, Kitaoka and Eto showed WIRL converges for multi-objective optimization \cite{Kitaoka-2023-convergence-IRL}*{Theorem 4.12}. 
To combine these theorem,
we get \cref{cor:intention_learning_complete_reward_2},
that is,
intention learning with WIRL enables the learner's reward values to imitate the expert's reward values in a finite number of iterations.
It means intention learning that WIRL is theoretically guaranteed to have a mechanism that frees users from manually designing the reward values.

\subsection*{Intention learning concerning action}
If the generator $g_{\phi}$ represents the expert's action, then when WIRL converges for multi-objective optimization, the learner's optimization actions coincide with the expert's actions \cref{theo:intention_learning_complete}. 
On the other hand, Kitaoka and Eto showed WIRL converges for multi-objective optimization \cite{Kitaoka-2023-convergence-IRL}*{Theorem 4.12}. 
To combine these theorems,
intentional learning with WIRL can theoretically be said to converge in the direction that the learner's actions imitate the expert's actions.

As a feature work, one question is whether intention learning with WIRL converges in a finite number of iterations.
Kitaoka and Eto showed WIRL converges for multi-objective optimization \cite{Kitaoka-2023-convergence-IRL}*{Theorem 4.12}. 
However, since it is not possible to actually try infinite iterations, it is necessary to guarantee that intention learning with WIRL converges in a finite number of iterations.
If we can show this, then intention learning with WIRL is theoretically guaranteed to have a mechanism that frees users from manually designing the action or solver.

\subsection*{Cases where expert actions are not represented by generators}

We raise some future works. Suppose the expert's actions are not represented by the generator $g_{\phi}$. In that case, it is interesting whether the learner's actions mimic the expert's actions when WIRL for multi-objective optimization converges. Ideally, the expert's actions would be represented by the generator $g_{\phi}$. In reality, however, writing down the expert's actions in a mathematical model is not always possible.

    

\bibliography{WIRL} 


\end{document}